%% file: main.tex
\documentclass{article}
\input{packs}

\input{symbs}

\usepackage{mathtools}

\DeclarePairedDelimiter\floor{\lfloor}{\rfloor}
\usepackage{authblk}
\usepackage{natbib}
\title{Linear Bandits with Feature Feedback}
\author[1]{Urvashi Oswal}
\author[2]{Aniruddha Bhargava\thanks{This research was performed when A.B. was at University of Wisconsin-Madison.}}
\author[1]{Robert Nowak}
\affil[1]{University of Wisconsin-Madison}
\affil[2]{Amazon}

\begin{document} 
\date{}
\maketitle

\begin{abstract} 
This paper explores a new form of the linear bandit problem in which the algorithm receives the usual stochastic rewards as well as stochastic feedback about which features are relevant to the rewards, the latter feedback being the novel aspect. The focus of this paper is the development of new theory and algorithms for linear bandits with feature feedback. We show that linear bandits with feature feedback can achieve regret over time horizon $T$ that scales like $k\sqrt{T}$, without prior knowledge of which features are relevant nor the number $k$ of relevant features. In comparison, the regret of traditional linear bandits is $d\sqrt{T}$, where $d$ is the total number of (relevant and irrelevant) features, so the improvement can be dramatic if $k\ll d$. The computational complexity of the new algorithm is proportional to $k$ rather than $d$, making it much more suitable for real-world applications compared to traditional linear bandits. We demonstrate the performance of the new algorithm with synthetic and real human-labeled data.
\end{abstract}

\section{Introduction}

Linear stochastic bandit algorithms are used to sequentially select actions to maximize rewards.  The linear bandit model assumes that the expected reward of each action is an (unknown) linear function of a (known) finite-dimensional feature associated with the action. Mathematically, if $\x_t \in \R^d$ is the feature associated with the action chosen at time $t$, then the stochastic reward is
\begin{align}\label{eqn:reward}
y_t = \x_t^\top \tstar +  \eta_t,
\end{align}
where $\tstar$ is the unknown linear functional and $\eta_t$ is a zero mean random variable. The goal is to adaptively select actions to maximize the rewards. This involves (approximately) learning $\tstar$ and exploiting this knowledge. Linear bandit algorithms that exploit this special structure have been extensively studied and applied  \cite{rusmevichientong10linearly,ay11improved}. Unfortunately, standard linear bandit algorithms suffer from the curse of dimensionality.  The regret
grows linearly with the feature dimension $d$. The dimension $d$ may be quite large in modern applications (\textit{e.g.}, $1000$s of features in NLP or image/vision applications). However, in many cases the linear function may only involve a sparse subset of the $k<d$ features, and this can be exploited to partially reduce dependence on $d$. In such cases, the regret of sparse linear bandit algorithms scales like $\sqrt{dk}$ \cite{Abbasi-Yadkori2012sparse,lattimore2018bandit}.

We tackle the problem of linear bandits from a new perspective that incorporates feature feedback in addition to reward feedback, mitigating the curse of dimensionality. Specifically, we consider situations in which the algorithm receives a stochastic reward {\em and} stochastic feedback indicating which, if any, feature-dimensions were relevant to the reward value. For example, consider a situation in which users rate recommended text documents and additionally highlight keywords or phrases that influenced their ratings. Figure~\ref{fig:feedback} illustrates the idea. Obviously, the additional ``feature feedback'' may significantly improve an algorithm's ability to home-in on the relevant features. The focus of this paper is the development of new theory and algorithms for linear bandits with feature feedback. We show that the regret of linear bandits with feature feedback scales linearly in $k$, the number of relevant features, without prior knowledge of which features are relevant nor the value of $k$. This leads to large improvements in theory and practice.

\begin{figure}[H]
\centering \includegraphics[width=0.75\textwidth]{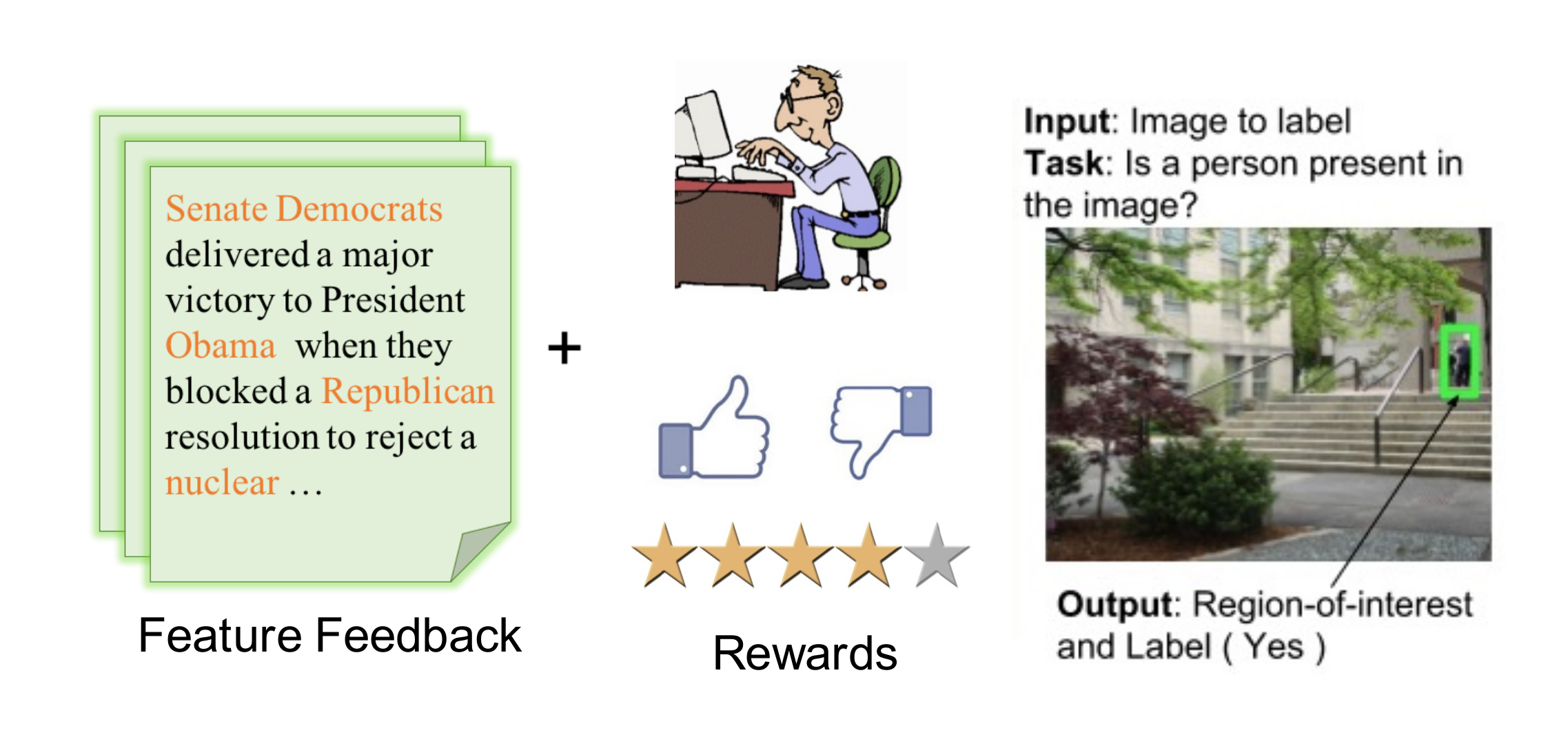} 
\caption{(Left) Highlighted words for text-based applications and (Right) Region-of-interest feature feedback for image-based applications.}
\label{fig:feedback}
\end{figure}

Perhaps the most natural and simple way to leverage the feature feedback is an explore-then-commit strategy. In the first $T_0$ steps the algorithm selects actions at random and receives rewards and feature feedback. If $T_0$ is sufficiently large, then the algorithm will have learned all or most of the relevant features and it can then switch to a standard linear bandit algorithm operating in the lower-dimensional subspace defined by those features. There are two major problems with such an approach:
\begin{enumerate}
\item The correct choice of $T_0$ depends on the prevalence of relevant features in randomly selected actions, which generally is unknown. If $T_0$ is too small, then many relevant features will be missed and the long-run regret will scale linearly with the time horizon. If $T_0$ is too large, then the initial exploration period will suffer excess regret. This is depicted in Figure~\ref{fig.two_stage_comparison}.
\item Regardless of the choice of $T_0$, the regret will grow linearly during the exploration period. The new FF-OFUL algorithm that we propose combines exploration and exploitation from the start and can lead to smaller regret initially and asymptotically as shown in Figure~\ref{fig.two_stage_comparison}.
\end{enumerate}
These observations motivate our proposed approach that dynamically adjusts the trade-off between exploration and exploitation.  A key aspect of the approach is that it is automatically adaptive to the unknown number of relevant features $k$. Our theoretical analysis shows that its regret scales like $k\sqrt{T}$. Experimentally, we show the algorithm generally outperforms traditional linear bandits and the explore-then-commit strategy. This is due to the fact that the dynamic algorithm exploits knowledge of relevant features as soon as they are identified, rather than waiting until all or most are found.  A key consequence is that our proposed algorithm yields significantly better rewards at early stages of the process, as shown in Figure~\ref{fig.two_stage_comparison} and in more comprehensive experiments later in the paper. The intuition for this is that estimating $\tstar$ on a fraction of the relevant coordinates can be exploited to recover a fraction of the optimal reward. Similar ideas are explored in linear bandits (without feature feedback) in \citet{deshpande2012linear}.

\begin{figure}[t!]
\centering \includegraphics[width=0.65\textwidth]{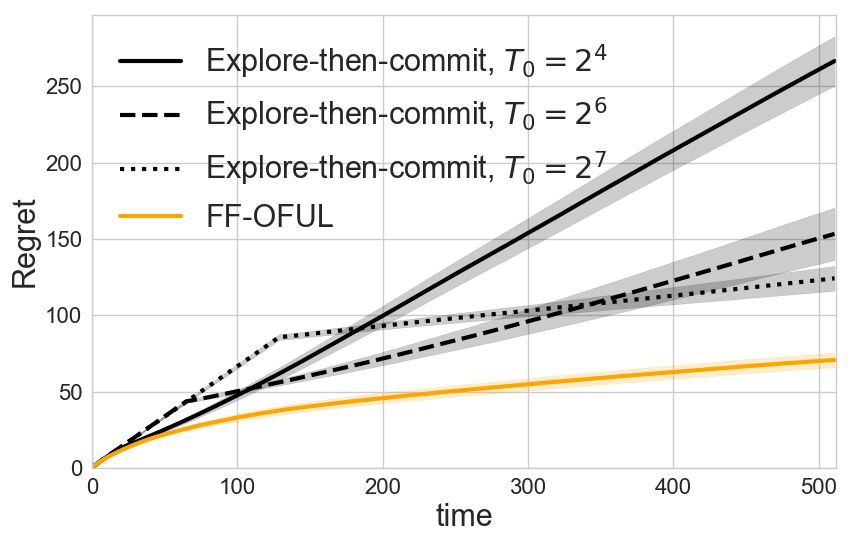} 
\caption{Comparison of the explore-then-commit strategy for different values of $T_0$ and our new FF-OFUL algorithm which combines exploration and exploitation steps (details of data generation in Section~\ref{sec:synthetic}).}
\label{fig.two_stage_comparison}
\end{figure}

\subsection{Motivating Application}

Consider the application of recommending news articles. At every time instant, the algorithm recommends an article to the user from a large database containing articles about topics like ``politics", ``technology", ``sports". The user provides a numerical reward corresponding to her assessment of the document's value. The goal of the algorithm is to maximize the cumulative reward over time. This can be challenging if the majority of the documents in the database are not of interest to the user. Linear bandit algorithms strike a balance between {\em exploration} of the database to ascertain the user's interests and {\em exploitation} by retrieving documents similar to those that have received the highest rewards (In the paper, we also refer to exploration-exploitation in the context of the knowledge of relevant features). Typical word models such as TF-IDF result in features ($d$) in the order of thousands of dimensions. The high-dimensionality makes it challenging to employ state-of-the-art algorithms since it involves maintaining and updating a $d \times d$ matrix at every stage. The approach taken in this work is to augment the usual reward feedback with additional feature feedback by allowing the user to highlight words or phrases to help orient the search. As an example, suppose the user is looking for articles about NFL football. They can highlight words such as ``Patriots", ``Football", ``Rams" to reinforce search in that direction and also negative words such as ``politics", ``stocks" to avoid in the document search. However words such as ``grass", ``air"  may be common words and therefore less relevant to the search. The goal is to give the user a tool to speed up their search with nearly effortless feedback.

\subsection{Definitions}

For round, $t$, let $\cX_t \subseteq \Real^d$ be the set of actions/items provided to the learner. We assume the standard linear model for rewards with a hidden weight vector $\tstar \in \Real^d$. If the learner selects an action, $\x_t \in \cX_t $, it receives reward, $y_t$, defined in (\ref{eqn:reward}) where $\eta_t$ is noise with a sub-Gaussian random distribution with parameter $R$.

For the set of actions $\cX_t$, the optimal action is given by, $\x^*_t \defeq \argmax_{\x \in \cX_t} \x^\top \tstar$, which is unknown. We define regret as,
\begin{equation}\label{eqn:regret}
R_T = \sum_{t=1}^T \left( \x^{*\top}_t \tstar - \x^{\top}_t \tstar \right).
\end{equation}
This is also called cumulative regret but, unless stated otherwise, we will refer to it as regret. We refer to the quantity $\x^{*\top}_t \tstar - \x^{\top}_t \tstar$ as the instantaneous regret which is the difference between the optimal reward and the reward received at that instant. We make the standard assumption that the algorithm is provided with an enormous action set which is only changing slowly over time, for instance, from sampling the actions without replacement ($\mathcal{X}_{t+1} = \mathcal{X}_t\backslash \x_t$).

\subsection{Related Work}

The area of contextual bandits was introduced by \citet{ginebra1995response}. The first algorithms for linear bandits appeared in \citet{abe99associative} followed by those using the optimism in the face of uncertainty principle, \citet{auer02using, dani2008stochastic}. \citet{rusmevichientong10linearly} showed matching upper and lower bounds when the action (feature) set is a unit hypersphere. Finally, \citet{ay11improved} gave a tight regret bound using new martingale techniques. We use their algorithm, OFUL, as a subroutine in our work. In the area of sparse linear bandits, regret bounds are known to scale like $\sqrt{kdT}$, \cite{Abbasi-Yadkori2012sparse, lattimore2018bandit}, when operating in a $d$ dimensional feature space with $k$ relevant features. The strong dependence on the ambient dimension $d$ is unavoidable without further (often strong and unrealistic) assumptions. For instance, if the distribution of feature vectors is isotropic or otherwise favorably distributed, then the regret may scale like $k\log(d)\sqrt{T}$, e.g., by using incoherence based techniques from compressed sensing \cite{carpentier2012bandit}. These results also assume knowledge of sparsity parameter $k$ and without it no algorithm can satisfy these regret bounds for all $k$ simultaneously.

In contrast, we propose a new algorithm that automatically adapts to the unknown sparsity level $k$ and removes the dependence of regret on $d$ by exploiting additional feature feedback. In terms of feature feedback in text-based applications, \citet{croft1989experiments} have proposed a method to reorder documents based on the relative importance of words using feedback from users. \citet{poulis2017learning} consider a similar problem but for learning a linear classifier. We use a similar feedback model but focus on the bandit setting where such feedback can be naturally collected along with rewards to improve search while striking a balance between exploration and exploitation leading to interesting tradeoffs. The idea of allowing user's to provide richer forms of feedback has been studied in the active learning literature \cite{raghavan2006active, druck2009active} and also been considered in other (interactive) learning tasks, such as cognitive science \cite{roads2016using}, machine teaching \cite{chen2018near}, and NLP tasks \cite{yessenalina2010multi}.

\section{Model for Feature Feedback}

The algorithm presents the user with an item (\textit{e.g.}, document) and the user provides feedback in terms of whether they like the item or not (logistic model) or how much they like it (inner product model). The user also selects a few features (\textit{e.g.}, words), if they can find them, to help orient the search. The reasonable assumption in the high-dimensional setting is that the linear bandit weight vector $\tstar$ is sparse (or approximately sparse). Suppose one is searching for articles about machine learning. It is easy to see how one may pay attention to words like pattern, recognition, and networks, but the vast majority of words may not help at all in determining if that article is about machine learning. 

\begin{ass}[Sparsity]\label{ass:sparsity}
The hidden weight vector $\tstar \in \Real^d$ is $k$-sparse and $k$ is unknown. In other words, $\tstar$ has at most $k$ non-zero entries or if $\supp(\tstar) = \{ i | {\tstar}_i \neq 0\}$ then $|\supp(\tstar)|= k \leq d$.
\end{ass}

Assumption~\ref{ass:sparsity} ensures that there are at most $k$ relevant features, however we stress that the value of $k$ is unknown (it is possible that all $d$ features are relevant). We make the following underlying assumptions about feature feedback.

\begin{ass}[Discoverability]\label{ass:discoverability}
For an action $\x \in \cX$ selected uniformly at random, the probability that a relevant feature is present and is selected is at least $p > 0$ (unknown).
\end{ass}

Assumption~\ref{ass:discoverability} ensures that while every item may not have relevant features, we are able to find them with a non-zero probability when searching through items at random. This assumption can be viewed as a (possibly pessimistic) lower bound on the rate at which relevant features are discovered. For example, it is possible that exploitative actions may yield relevant features at a higher rate (e.g., relevant features may be correlated with higher rewards). We do not attempt to model such possibilities since this would involve making additional assumptions that may not hold in practice.

\begin{ass}[Noise]\label{ass:nonoise}
Users may report irrelevant features. The number of reported irrelevant features (denoted by $0\leq k' \leq d-k$) is unknown in advance. 
\end{ass}

Assumption~\ref{ass:nonoise} accounts for ambiguous features that are irrelevant but users erring on the side of marking as relevant.

The set up is as follows: we have a set of items or actions, $\cX \subseteq \Real^d$ that we can propose to the users. There is a hidden weight vector $\tstar \in \Real^d$ that is $k$-sparse. We will further assume that $\lV \tstar \rV \leq S$ and the action vectors are bounded in norm: $\forall \x \in \cX, \lV \x \rV \leq L$.  Besides the reward $y_t$, defined in (\ref{eqn:reward}), at each time-step the learner gets $\mathcal{I}_t \subseteq \supp(\tstar)$ which is the relevance feedback information. The model further specifies that $\forall j \in \supp(\tstar), \Pr(j \in \mathcal{I}_t) \geq p$. That is, the probability a relevant feature is selected at random is at least $p$. We need this assumption to make sure that we can find all the relevant features.

\section{Algorithm}

In this section, we introduce an algorithm that makes use of feature relevance feedback in order to start with a small feature space and gradually increase the space over time without the knowledge of $k$. We begin by reminding ourselves of the following theorem that bounds the regret (\ref{eqn:regret}) of the OFUL algorithm (stated as  Algorithm~\ref{alg:OFUL}) based on the principle of optimism in the face of uncertainty. The algorithm constructs ellipsoidal confidence sets centered around the ridge regression estimate, using observed data such that the sets contain the unknown $\theta_{\ast}$ with high probability, and selects the action/item that maximizes the inner product with any $\theta$ from the confidence set. 

\begin{thm}[\citet{ay11improved}]\label{thm:OFULRegret}
Assume that $\forall t > 0$ and $\x \in \cX_t \subset \R^d$, $\la \x, \tstar \ra \in [-1,1]$. Then with probability at least $1 - \delta$, the regret of OFUL satisfies:
\begin{align*} 
\forall t, R_t \leq & 4 \sqrt{td \log(\lambda + tL/d)} ( \lambda^{1/2} S + R \sqrt{2 \log(1/\delta) + d \log(1 + tL/(\lambda d))} ) 
\end{align*}
where $\lambda > 0$ is the ridge regression parameter of OFUL.
\end{thm}

\begin{algorithm}[t!]
{
  \begin{algorithmic}[1]
    \FOR {$t = 1,2, \ldots, T-1$}
    \STATE $(\x_t, \tilde{\tta}_t) = \argmax_{(\x, \tta) \in \cX_t \times \cC_{t-1}} \la \x, \tta \ra$
    \STATE Select action $\x_t$ and receive reward $y_t$.
    \STATE Update $\overline{\V}_t =(\X_t^T \X_t + \lambda \I)$ and $\hth_t = \overline{\V}_t ^{-1} \X_t^T \y_t$ 
    \STATE Update ellipsoidal confidence set $\cC_t$ as \\
     \small{$\cC_t = \left\{ \th \in \R^d : \|\hth_t - \th\|_{\overline{\V}_t} \leq f(R, \X_t, \y_t, \lambda, \delta, S) \right\}$}\\
     (for details on $f(\cdot)$ see \cite{ay11improved})
    \ENDFOR
  \end{algorithmic}
  \caption{OFUL from~\cite{ay11improved}}
  \label{alg:OFUL}
}
\end{algorithm}

Roughly, this theorem provides a $\tilde{O}(d\sqrt{t})$ bound on the regret of OFUL stated as Algorithm~\ref{alg:OFUL} by ignoring constants and logarithmic terms. We will combine this with a form of $\epsilon$-greedy algorithm due to~\cite{sutton1998reinforcement} to prove a result similar to Theorem~\ref{thm:OFULRegret} but reduce the dependence on the dimension from $d$ to $k$. 

In order to do so, we must discover the support of $\tstar$. The idea being that we apportion a set of actions to random plays in order to guarantee that we find all the relevant features, and the remaining time we will run OFUL on the identified relevant dimensions. Reducing the proportion of random actions over time guarantees that the regret remains sub-linear in time. We propose Algorithm~\ref{alg:contHypOFUL} to exploit feature feedback. Here, at each time $t$, with probability proportional to $1/\sqrt{t}$, the algorithm selects an action/item to present at random, otherwise it selects the item recommended by feature-restricted-OFUL. 

All updates are made only in the dimensions that have been marked as relevant and the space is dynamically increased as new relevant features are revealed. If nothing is marked as relevant, then by default the actions are selected at random, potentially suffering the worst possible reward but, at the same time, increasing our chances of getting relevance feedback leading to a trade-off. As time goes on, more relevance information is revealed. Note that the algorithm is adaptive to the unknown number of relevant features $k$. If $k$ were known, we could stop looking for features when all relevant ones have been selected.  We find that in practice, this algorithm has an additional benefit of being more robust to changes in the ridge parameter ($\lambda$) due to its intrinsic regularization of restricting the parameter space.

\begin{algorithm}[t!]
{
  \begin{algorithmic}[1]
    \STATE Let the set of relevant indices, $\mathcal{R}_0, \mathcal{I}_0 = \{\}$.
    \WHILE{$\mathcal{I}_0 $ is empty}
    	\STATE Select action at random, $\mathcal{I}_0 = \{$ indices revealed $\} $
    \ENDWHILE
    \STATE $\mathcal{R}_{1} = \mathcal{R}_0 \bigcup \mathcal{I}_0$
    \STATE Initialize $\cC_0$ using actions sampled.
    \FOR {$t=1,2, \ldots, T$}
    \STATE Let $\mX_t$ be the feature matrix restricted to $\mathcal{R}_{t}$.
    \STATE Set $\epsilon_t = 1/\sqrt{t}$. Draw $b_t$ from bernoulli$(\epsilon_t)$
     \IF{$b_t = 1$}
        \STATE Pick an action $\x_t$ uniformly at random from $\mathcal{X}_t$, 
        \ELSE 
        \STATE Pick $(\x_t, \tilde{\tta}_t) = \argmax_{(\x, \tta) \in \cX_t \times \cC_{t-1}} \la \x, \tta \ra$
        \ENDIF
        \STATE With action $\x_t$ observe reward $y_t$ and indices, $\mathcal{I}_t$.
         \STATE Update $\mathcal{R}_{t} = \mathcal{R}_t \bigcup \mathcal{I}_t$
        \IF{$\mathcal{I}_t$ is empty}
        \STATE Rank one update to $\overline{\V}_t, \hth_t, \cC_{t}$ (see Algorithm~\ref{alg:OFUL}) using $(y_t, \x_t)$
        \ELSE
        \STATE Update $\mX_t$ with features in $\mathcal{R}_t$.  
        \STATE Recompute $\overline{\V}_t, \hth_t, \cC_{t}$ with new feature set $\mX_t$. 
        \ENDIF

    \ENDFOR
  \end{algorithmic}
  \caption{Feature Feedback OFUL (FF-OFUL)}
  \label{alg:contHypOFUL}
}
\end{algorithm}

\section{Regret Analysis}

In this section, we state the regret bounds for the FF-OFUL algorithm along with a sketch of the proof and discussion on approaches to improve or generalize the bounds. The more subtle proof details are deferred to the appendix.

\subsection{Regret Bound for Algorithm~\ref{alg:contHypOFUL} (FF-OFUL)}
Recall that the norm of the actions are bounded by $L$ and the hidden weight vector $\tstar \in \R^d$ is also bounded in norm by $S$. Therefore, for any action, the worst-case instantaneous regret can be derived using Cauchy-Schwarz as follows:
\begin{align*}
\lvert\left\langle \x^*, \tstar \right\rangle - \left\langle \x, \tstar  \right\rangle \rvert \quad & \leq \lvert  \left\langle \x^*, \tstar \right\rangle \rvert +\lvert \left\langle \x, \tstar \right\rangle \rvert \quad\\
& \leq \lV \x^* \rV \lV \tstar \rV + \lV \x \rV \lV \tstar \rV \quad \\
& \leq 2 S L
\end{align*}
We provide the main result that bounds the regret (\ref{eqn:regret}) of Algorithm~\ref{alg:contHypOFUL} in the following theorem.

\begin{thm}\label{thm:HybOFULRegret}
With the same assumptions as Theorem~\ref{thm:OFULRegret}: $\forall t >  0$ and $\x \in \cX_t$, $\la \x, \tstar \ra \in [-1,1]$, if we have the additional assumptions~\ref{ass:sparsity}, \ref{ass:discoverability} and \ref{ass:nonoise} (k' = 0). Then with probability at least $1 - \delta$, the cumulative regret after $T$ time steps for Algorithm~\ref{alg:contHypOFUL} is:
\begin{align*}
R_T & \leq \frac{8 S L}{\log 6M/\delta} \left( \frac{\log 3k/\delta}{\log 1/(1-p)} \right)^2  + \log_2\frac{T}{2} \left( 3 S L \sqrt{T \log \frac{6M}{\delta}}\right) + 4 \log_2\frac{T}{2} \\
&  \sqrt{\frac{T}{2} k \log(\lambda + n L/k) } \left( \lambda^{1/2} S\right. + \left. R \sqrt{2 \log(3M/\delta) + k \log(1 + T L/(2 \lambda k))} \right).
\end{align*}
where $M =\log_2\frac{T}{2}$, $\lambda > 0$ is the ridge regression parameter and $k$ is the (unknown) number of relevant features.
\end{thm}
 
In other words, with high probability, the regret of Algorithm~\ref{alg:contHypOFUL} (FF-OFUL) scales like $\tilde{O}(k \sqrt{T} + \frac{1}{p^2})$ by ignoring constants and logarithmic terms and using the taylor series expansion of $-\log (1-p)$. The three terms in the total regret come from the following events. Regret due to: 
\begin{enumerate}[leftmargin=0cm,itemindent=.5cm,labelwidth=\itemindent,labelsep=0cm,align=left]
\item Exploration to guarantee observing all the relevant features (with high probability).
\item Exploration after observing all relevant features (due to lack of knowledge of $p$ or $k$).
\item Exploitation and exploration running OFUL (after having observed all the relevant features).
\end{enumerate}

In practice, feature feedback may be noisy. Sometimes, features that are irrelevant may be marked as relevant. To account for this, we can relax our assumption to allow for subset of $k^\prime$ irrelevant features that are mistakenly marked as relevant. Including these features will increase the regret but the algorithm will still work and the theory goes through without much difficulty as stated in the following corollary. 

\begin{cor}
With the same assumptions as Theorem~\ref{thm:HybOFULRegret}, if we assume that a fixed set of $k^\prime$ irrelevant features were indicated by the user (Assumption~\ref{ass:nonoise}), then the regret of Algorithm~\ref{alg:contHypOFUL} (FF-OFUL) scales like $\tilde{O}((k+k^\prime) \sqrt{T} + \frac{1}{p^2})$.
\label{cor:noisy_feedback}
\end{cor}

The corollary follows by observing that the exploration is not affected by this noise and the regret of exploitation on the weight vector restricted to the $k+k^\prime$ dimensions scales like $(k+k^\prime) \sqrt{T}$. This accounts for having some features being ambiguous and users erring on the side of marking them as relevant. This only results in slightly higher regret so long as $k+k^\prime$ is still smaller than $d$. One could improve this regret by making additional assumptions on the probabilities of feature selection to weed out the irrelevant features.

\subsubsection{Proof Sketch of Main Result} 

We provide a sketch of the proof here and defer the details to the appendix. Recall, the cumulative regret is summed over the instantaneous regrets for $t = 1,\dots, T$. We divide the cumulative regret across epochs $s = 0, \dots, M$ of doubling size $T_s = 2^s$ for $M = \log_2\frac{T}{2}$. 
\begin{figure}[h]
\centering \includegraphics[width=0.75\textwidth]{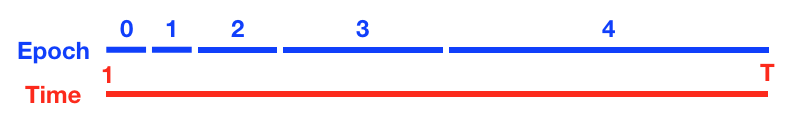} 
\caption{Time horizon divided into epochs of doubling size.}
\label{feedback}
\end{figure}

This ensures that the last epoch dominates the regret which gives the multiplicative factor of $\log_2\frac{T}{2}$. For each epoch, we bound the regret under two events, all relevant features have been identified (via user feedback) up to that epoch or not. First, we bound the regret conditioned on the event that all the relevant features have been identified in Lemma~\ref{lem:ff_epoch_regret}. This is further, in expectation, broken down into the $\epsilon_s$ portion of random actions for pure exploration (Lemma~\ref{lem:num_rand_arms}) and $1- \epsilon_s$ modified OFUL actions on the $k$-dimensional feature space for exploitation-exploration (Lemma~\ref{lem:oful_extended}). For the pure exploration part, we use the worst case regret bound but since $\epsilon_s$ is decreasing this does not dominate the OFUL term. Second, we bound the probability that some of the relevant features are not identified so far (Proposition~\ref{prop:prob_of_no_obs}), which is a constant depending on $k$ and $p$ since it becomes zero after enough epochs have passed. We need pure exploration to ensure the probability that some features are not identified decreases with each passing epoch.  The regret in this case is bounded with worst case regret. 

A subtle issue of bounding regret of the actions selected by the OFUL subroutine is that, unlike OFUL, the confidence sets in our algorithm are constructed using additional actions from exploration rounds and past epochs. To accommodate this we prove a regret bound for this variation in Lemma~\ref{lem:oful_extended}. Putting all this together gives us the final result.

\textbf{Lower bound.}
We can use the arguments from~\citep{dani2008stochastic, rusmevichientong10linearly} to get a lower bound of $O(k \sqrt{T})$. To see this, assume that we know the support. Then any linear bandit algorithm that is run on that support must incur an order $k \sqrt{T}$ regret. We don't know the support but we estimate it with high probability and therefore the lower bound also applies here. Our algorithm is optimal up to log factors in terms of the dimension.

\subsection{Better Early-Regret Bounds}\label{sec.early_regret}

In our analysis, we bound the regret in the rounds before observing all relevant features with the worst case regret. This may be too pessimistic in practice. We present some results to support the idea of restricting the feature space in the short-term horizon and growing the feature space over time. The results also suggest that an additional assumption on the behavior of early-regret could lead to better constants in our bounds. Any linear bandit algorithm restricted to the support of $\tstar$ must incur an order $k \sqrt{T}$ regret so one can only hope to improve the constants of the bound. 

Figure~\ref{fig:sub_inds_synth}(a) shows that the average regret of pure exploration has a worse slope than that of OFUL restricted to a subset of the relevant features. We randomly sampled $N=1000$ actions from the unit sphere in $d = 40$ dimensions and generated $\theta_{\ast}$ with $k = 5$ sparsity. The only regret bound one can derive for a pure exploration algorithm that picks actions uniformly at random, independent of the problem instance, is a worst-case cumulative regret bound of $2SLT$. Let $R^{\mathcal{K}}_{alg}$ be the expected regret of algorithm $alg$ run on the subset of relevant features $\mathcal{K}\subseteq \{1, \ldots, k\}, \lvert \mathcal{K} \rvert = j \leq k$. For example, $alg$ could be the OFUL algorithm. Then $R^{\mathcal{K}}_{alg}$ represents the expected regret of OFUL only restricted to features in $\mathcal{K}$. 
Suppose the explore-then-commit algorithm first explores for roughly $\sqrt{T}$ time instances to discover relevant features ($\mathcal{K}$) followed by an exploitation stage such as OFUL only restricted to features in $\mathcal{K}$. The rewards in the exploitation stage can be divided in two parts, 
\begin{align*}
\left\langle \x, \tstar  \right\rangle =  \left\langle \x^\mathcal{K}, \tstar^\mathcal{K}  \right\rangle  +  \left\langle \x^{\mathcal{K}^c}, \tstar^{\mathcal{K}^c}  \right\rangle,
\end{align*}
where $\x^\mathcal{K}$ is the portion of $\x$ restricted to $\mathcal{K}$ and $\mathcal{K} \cup \mathcal{K}^c = [p]$. Similarly, the regret $R^{\mathcal{K}}_{alg}$ can be divided in two parts. Roughly the regret on $\mathcal{K}$ can be bounded by $j\sqrt{T}$ under certain conditions using the OFUL regret bound. For the regret on $\mathcal{K}^c$, suppose each relevant component of $\tstar$ has a mean square value of $S^2/k$ (for example, this can be achieved with a sparse gaussian model such as those described in \citet{deshpande2012linear}). This yields $\E \|\tstar^{\mathcal{K}^c}\|^2 \approx \frac{k-j}{k} S^2$ where $j = \lvert \mathcal{K} \rvert$. The worst-case instantaneous regret bound on $\mathcal{K}^c$ becomes $2 \sqrt{\frac{k-j}{k}}SL$ leading to an improvement in the slope of linear regret by a factor of $\sqrt{\frac{k-j}{k}}$ over pure exploration as seen in Figure~\ref{fig:sub_inds_synth}. 

\begin{figure*}[t!]
  \centering
  \begin{minipage}[b]{0.45\textwidth}
  \centering
    \includegraphics[width=0.95\textwidth]{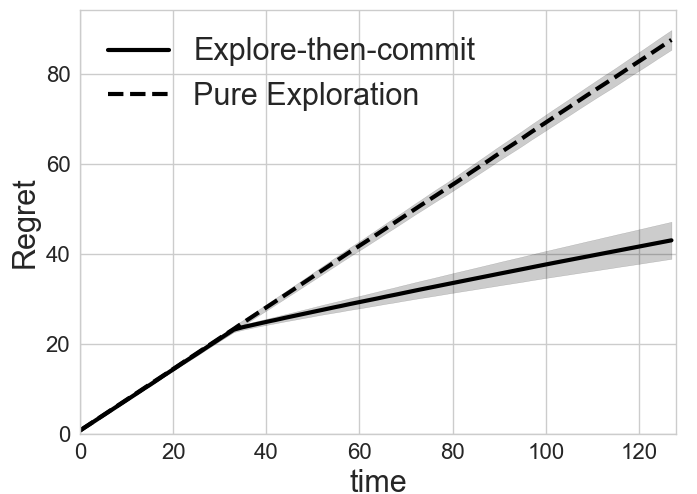}
    \centerline{\small{(a) Synthetic data}}
  \end{minipage}
  \begin{minipage}[b]{0.45\textwidth}
  \centering
    \includegraphics[width=0.95\textwidth]{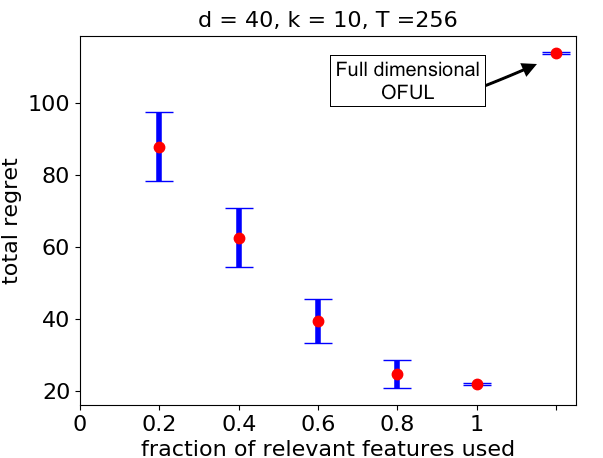}
    \centerline{\small{(b) Synthetic data}}
  \end{minipage}
  \begin{minipage}[b]{0.95\textwidth}
  \centering
    \includegraphics[width=0.5\textwidth]{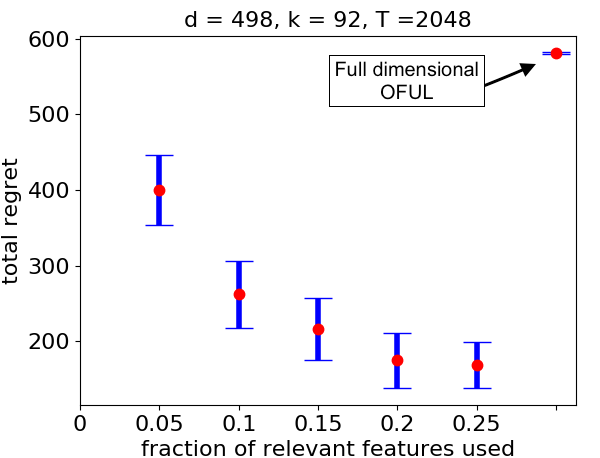}
    \centerline{\small{(c) Data from~\citet{poulis2017learning}}}
  \end{minipage}
   \caption{(a) Regret of pure exploration versus an explore-then-commit strategy (b,c) Average regret of OFUL restricted to feature subsets (red dots) with 95\% confidence regions (blue lines). It is important to note that a short time horizon was chosen to make the case for restricting the feature space in early rounds. In the long horizon, with more information, the relative performance of OFUL would improve, but would ultimately be a factor of $d/k$ worse than that of the low-dimensional model that includes all $k$ relevant features. } \label{fig:sub_inds_synth}
\end{figure*}

Figure~\ref{fig:sub_inds_synth}(b) shows the average regret of OFUL restricted to feature subsets of different sizes with synthetic data with $N=1000$ actions, $d = 40$ and $k = 10$. For $j \in \{2, 4, \dots, 10\}$, we randomly picked $100$ subsets of size $j$ from the support of $\tstar$. We report the average regret of OFUL for a short horizon, $T = 2^8$, restricted to $100$ random subsets. We also plot average regret of OFUL on the full $d = 40$ dimensional data. Figure~\ref{fig:sub_inds_synth}(c) depicts the same with real data from~\cite{poulis2017learning} with $d = 498$ and sparsity, $k = 92$, we choose $100$ random subsets of size $j \in \{5, 10, \dots, 25\}$ from the set of relevant features marked by users (see Section~\ref{sec:experiments} for more details). We report the average regret of OFUL restricted to the features from $100$ random subsets for a relatively short time horizon, $T = 2^{11}$.

The plots show that, in the short horizon, it may be more beneficial to use a subset of the relevant features than using the total feature set which may include many irrelevant features. The intuition is that when OFUL has not seen many samples, it does not have enough information to separate the irrelevant dimensions from relevant ones. As time goes on (i.e., for longer horizons) OFUL's relative performance improves since it enjoys sublinear regret but would ultimately be a factor of $d/k$ worse than that of the low-dimensional model that includes all $k$ relevant features. 

\section{Experiments}\label{sec:experiments}

In this section, we demonstrate the performance of our algorithm with synthetic and real human-labeled data. 

\subsection{Results with Synthetic Data}\label{sec:synthetic}
For the synthetic dataset, we simulate a text categorization dataset as follows. Each action corresponds to an article. Generally an article contains only a small subset of the words from the dictionary. Therefore, to simulate documents we generate $1000$ sparse actions in $40$ dimensions. A $5$-sparse reward generating vector, $\tstar$, is chosen at random. This is representative of the fact that in reality a document category probably contains only a few relevant words. The features represent word counts and hence are always positive. Here we have access to $\tstar$ therefore for any action $\x$, we use the standard linear model (\ref{eqn:reward}) for the reward $y_t$ with $\eta_t \sim \mathcal{N}(0, R^2)$. The support of $\tstar$ is taken as the set of oracle relevant words. For every round, each word from the intersection of the support of the action and oracle relevant words is marked as relevant with probability $p ( = 0.1)$. Figure~\ref{fig:ff_synth_res}(a) shows the results of an average of 100 random trials where $\tstar$ is sparse with $k = 5, d = 40$, with $1000$ actions. As expected, the FF-OFUL algorithm outperforms standard OFUL significantly. Figure~\ref{fig:ff_synth_res}(b) also shows that the feedback does not hurt the performance much for non-sparse $\tstar$ with $k = d = 40$. Figure~\ref{fig.two_stage_comparison} compares the performance of FF-OFUL with an explore-then-commit strategy.

\begin{figure*}[t!]

\begin{minipage}[b]{0.5\linewidth}
  \centering
  \centerline{\includegraphics[width=0.95\linewidth]{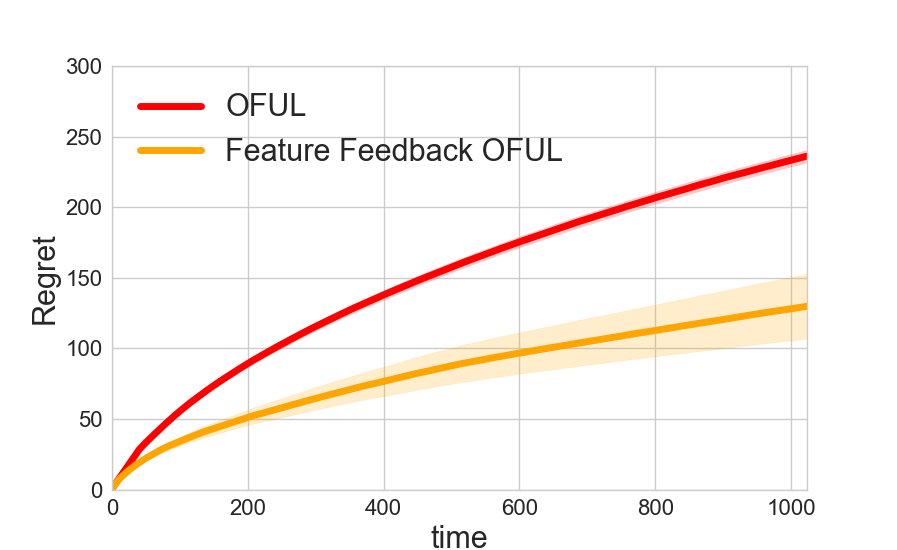}}
  \centerline{(a) Simulated data with sparse $\tstar$.}\medskip
\end{minipage}
\begin{minipage}[b]{0.5\linewidth}
  \centering
  \centerline{\includegraphics[width=0.95\linewidth]{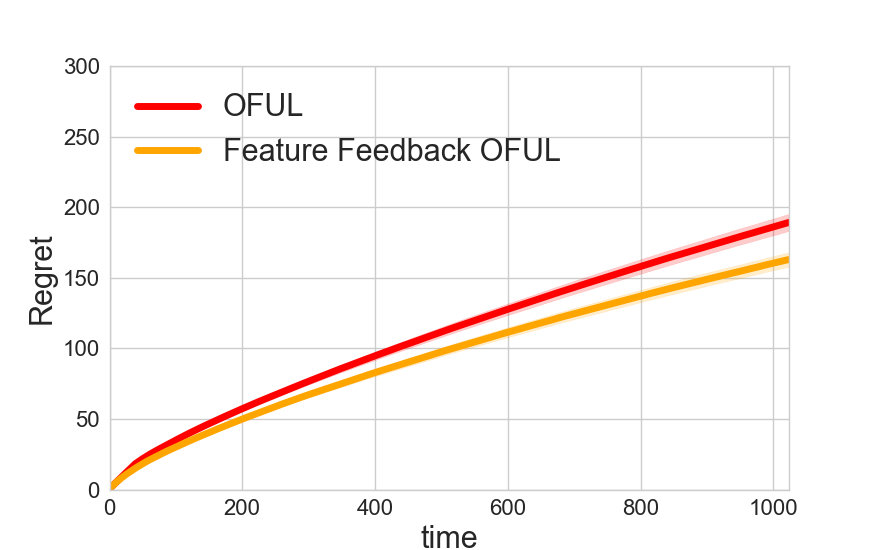}}
  \centerline{(b) Simulated data with dense $\tstar$.}\medskip
\end{minipage}
\caption{ On simulated data ($d = 40$) \textbf{(a)} with sparse $\tstar$ ($k = 5$), the algorithm using Feature Feedback outperforms OFUL significantly and \textbf{(b)} with dense $\tstar$ ($k = d = 40$), Feature Feedback does not hurt the performance and it is close to standard OFUL.\label{fig:ff_synth_res} Refer introduction for explore-then-commit comparison on synthetic data.}
\end{figure*}

\begin{figure*}[t!]
\begin{minipage}[b]{0.5\linewidth}
  \centering
  \centerline{\includegraphics[width=0.95\linewidth]{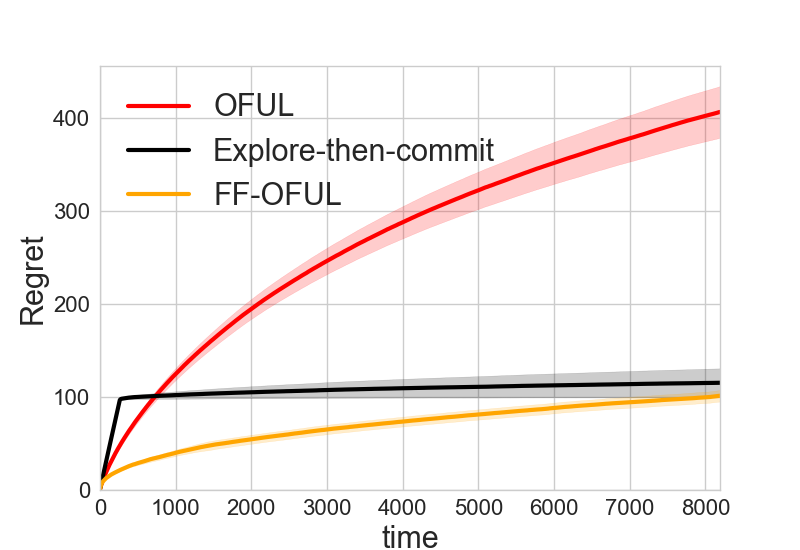}}
\centerline{(a) with replacement}\medskip
\end{minipage}
\begin{minipage}[b]{0.5\linewidth}
  \centering
  \centerline{\includegraphics[width=0.95\linewidth]{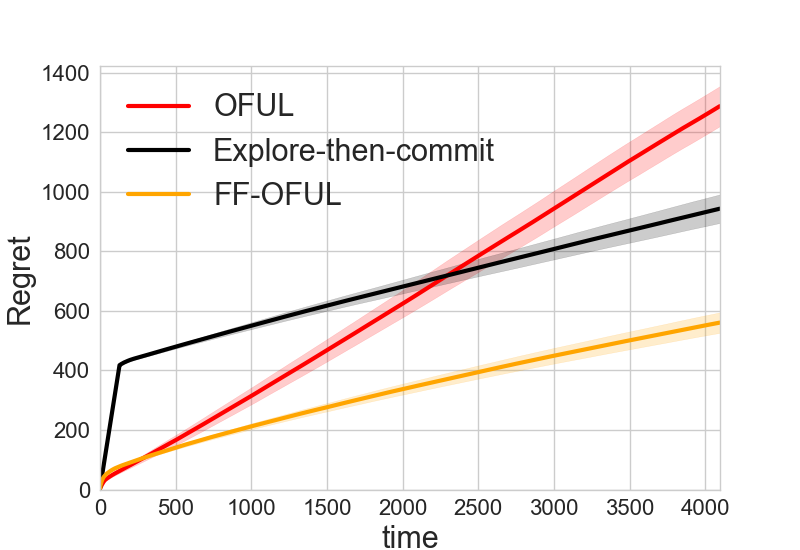}}
\centerline{(b) without replacement}\medskip
\end{minipage}
\caption{ Newsgroup dataset with oracle feedback: This plot shows that FF-OFUL outperforms OFUL and an Explore-then-commit strategy when running in $d = 1000$ dimensions, (Left) sampling actions with replacement using binary rewards model. \label{fig:ff_NG1000} (Right) sampling actions without replacement and using the numerical reward model. Smallest $T_0$ selected such that all relevant features are marked with high probability. Note that $T$ must be less than the number of actions for without replacement sampling hence the shorter time horizon.}
\end{figure*}

\begin{figure*}[t!]
\begin{minipage}[b]{0.5\linewidth}
  \centering
  \centerline{\includegraphics[width=\linewidth]{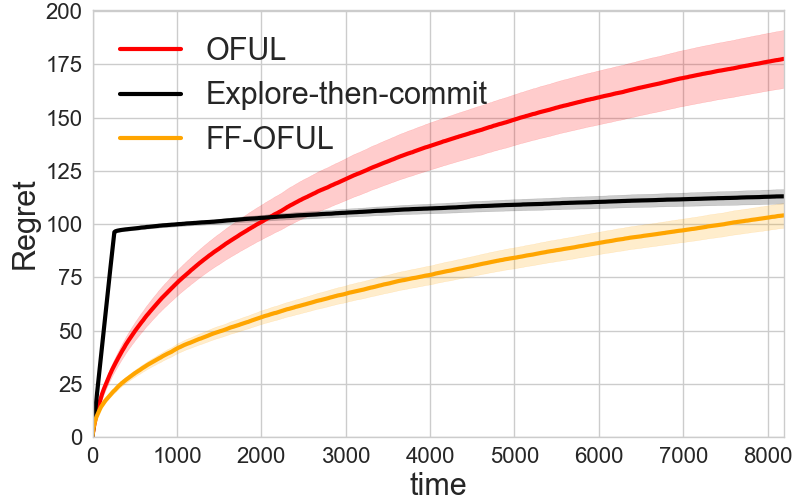}}
\centerline{(a)}\medskip
\end{minipage}
\begin{minipage}[b]{0.5\linewidth}
  \centering
  \centerline{\includegraphics[width=\linewidth]{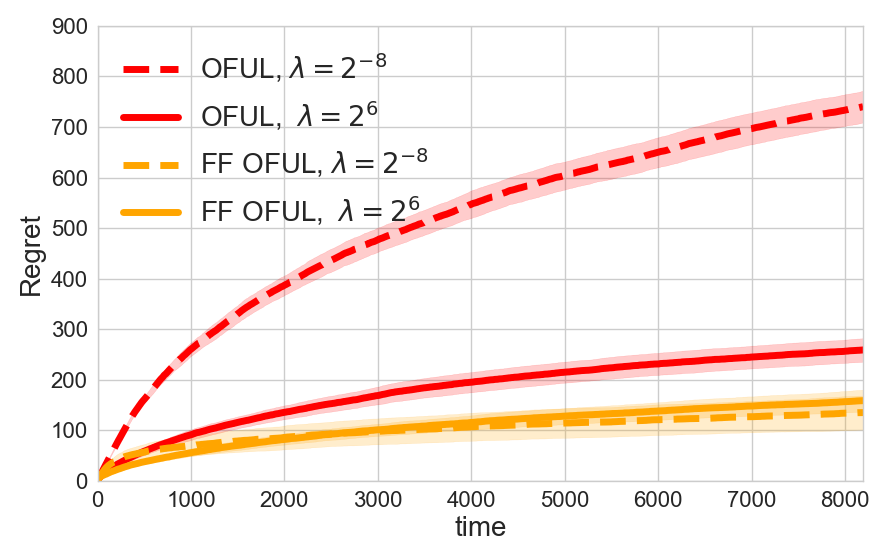}}
 \centerline{(b)}\medskip
\end{minipage}
\begin{minipage}[b]{0.99\linewidth}
  \centering
  \centerline{\includegraphics[width=0.5\linewidth]{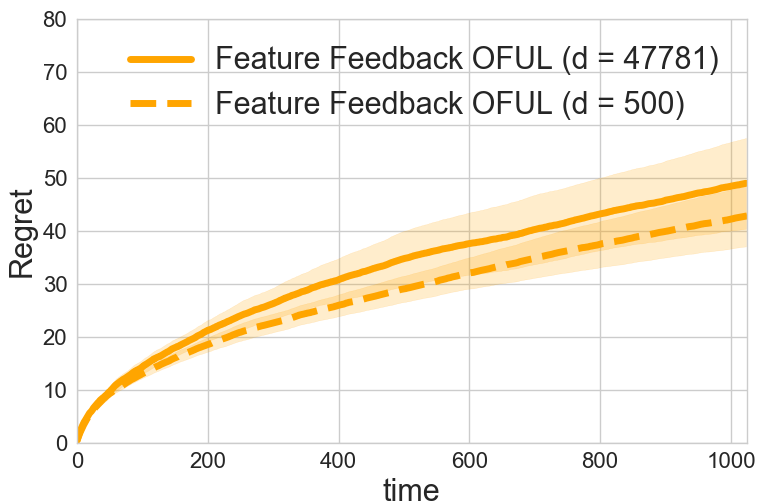}}
\centerline{(c)}\medskip
\end{minipage}
\caption{Newsgroup Dataset with Human Feedback: \textbf{(a)} This plot shows that FF-OFUL outperforms OFUL and an Explore-then-commit strategy in $d = 500$ dimensions. Both plots were generated by tuning the parameter for OFUL.\label{fig:ff_dasgupta500} \textbf{(b)}  Sensitivity to the tuning parameter $\lambda$ is seen by the drastic difference in performance of OFUL. In contrast, our proposed algorithm has a relatively modest difference in performance showing its robustness to the ridge regression parameter $\lambda$. \label{fig:ff_tuning_comparison} \textbf{(c)} Our algorithm for $d=47781$ and $d=500$ with the same ridge parameter $\lambda=1$, showing its robustness to changes in dimensions and tuning. \label{fig:ff_dasgupta_full}}
\end{figure*}

\subsection{Results with 20Newsgroup Dataset}
For real data experiments, we use the 20Newsgroup dataset from \cite{lang1995newsweeder}. It has $20,000$ documents covering $20$ topics such as politics, sports. We choose a subset of 5 topics (misc.forsale, rec.autos, sci.med, comp.graphics, talk.politics.mideast) with approximately 4800 documents posted under these topics. For the word counts, we use the TF-IDF features for the documents which give us approximately $d = 47781$ features. For the sake of comparing our method with OFUL, we first report $500$ and $1000$ dimensional experiments and then on the full $47,781$ dimensional data. To do this, we use logistic regression to train a high accuracy sparse classifier to select $153$ features. Then select an additional $847$ features at random in order to simulate high dimensional features. We compared OFUL and FF-OFUL algorithms on this data. This is similar to the way \citet{poulis2017learning} ran experiments in the classification setting. We ran only our algorithm on the full $47781$ dimension data since it was infeasible to run OFUL. For the reward model, we pick one of the articles from the database at random as $\tstar$ and  the linear reward model in (\ref{eqn:reward}) or use the labels to generate binary, one vs many rewards to simulate search for articles from a certain category. In order to come close to simulating a noisy setting, we used the logistic model, with $q_t = 1/(1 - \exp(-\la \x_t, \tstar \ra), P(y_t = +1) = q_t$.

\subsubsection{Oracle Feedback.} We used the support of the one vs many sparse logistic regression to get an ``oracle set of relevant features" for each class. Each word from the intersection of the support of the action and oracle relevant words was marked as relevant with probability $p ( = 0.1)$. There were about $k \in (30, 100)$ relevant features for each category. Figure~\ref{fig:ff_NG1000}, shows the performance of OFUL, Explore-then-commit and FF-OFUL on the Newsgroup dataset with oracle feedback. In these simulations averaged over $100$ random $\tstar$, FF-OFUL outperforms OFUL and Explore-then-commit significantly. OFUL parameter was tuned to $\lambda = 2^8$.  

\subsubsection{Human Feedback.}  \citet{poulis2017learning} took $50$ of the 20Newsgroup articles from $5$ categories and had users annotate relevant words. These are the same categories that we used in the Newsgroup20 results. This is closer to simulating human feedback since we are not using sparse logistic regression to estimate the sparse vectors. We take the user indicated relevant words instead as the relevance dimensions. . There were $k \in (30, 100)$ relevant features for each category. In Figure~\ref{fig:ff_dasgupta500}(a), we can see that  FF-OFUL is already outperforming OFUL and Explore-then-commit. This is despite the fact that it is not a very sparse regime. Surprisingly, we found that tuning had little effect on the performance of  FF-OFUL whereas it had a significant effect on OFUL (see Figure~\ref{fig:ff_tuning_comparison}). We believe that this behavior is due to the gradual growth in the number of relevant dimensions as we receive new feedback therefore implicitly regularizing the number of free parameters. FF-OFUL also yields significantly better rewards at early stages by exploiting knowledge of relevant features as soon as they are identified, rather than waiting until all or most are found.

\subsubsection{Parameter Tuning.}
For OFUL we tune the ridge parameter ($\lambda$) in the range $\{2^{-7}, 2^{-6}, \ldots, 2^{10} \}$ to pick the one with best performance. All the tuned parameters that were selected for OFUL were strictly inside this range. For $d=40, k = 5$ and $k = 40$, $\lambda_{OFUL} = 2^{-5}$. For $d = 1000$ (Newsgroup), $\lambda_{OFUL} = 2^8$. Figure~\ref{fig:ff_tuning_comparison}(b) demonstrates the sensitivity of OFUL to change in tuning parameter. For FF-OFUL, the remarkable feature is that it does not require parameter tuning so $\lambda = 1$ for all experiments.

\subsubsection{Full dimension experiments.} Remarkably the performance of our algorithm barely drops in full ($d = 47781$) feature dimensions as seen in Figure~\ref{fig:ff_dasgupta_full}(c). It is important to note that the ridge regression parameter ($\lambda$) for all the experiments was set to $\lambda = 1$ and was not tuned. FF-OFUL is robust to changes in the ambient dimensions and the parameter $\lambda$. Recall that we do not compare the results with OFUL on $47781$ dimensional data since it would require storing and updating a $47781 \times 47781$ matrix at each stage.

\section{Conclusion}

In this paper we provide an algorithm that incorporates feature feedback in addition to the standard reward feedback. We would like to underline that since this algorithm incrementally grows the feature space, it makes it possible to use the new algorithm in high-dimensional settings where conventional linear bandits are impractical and also makes it less sensitive to the choice of tuning parameters. This behavior could be beneficial in practice since tuning bandit algorithms could be sped up. In the future, it might prove fruitful to augment the feature feedback provided by the user with ideas from compressed sensing to facilitate faster recognition of relevant features.

\bibliography{main}
\bibliographystyle{apalike}

\appendix

\section{Feature Feedback Epoch OFUL}
This second algorithm, Feature Feedback Epoch OFUL (Algorithm~\ref{alg:epochHybOFUL}), is an epoch version of Algorithm~\ref{alg:contHypOFUL} which runs in epochs of doubling length so the last epoch dominates the regret. It is essentially the same as Algorithm~\ref{alg:contHypOFUL} written in a different format which facilitates proving the main result. The main difference in the algorithms is the choice of $\epsilon_t$ depicted in Figure~\ref{fig:epsilon_greedy}.

\begin{algorithm}[H]
{\small
  \begin{algorithmic}[1]
    \STATE Let the set of relevant indices, $\mathcal{R}_0, \mathcal{I}_0 = \{\}$.
    \WHILE{$\mathcal{I}_0 $ is empty}
    	\STATE Pull arm at random, $\mathcal{I}_0 = \{$ indices revealed $\} $
    \ENDWHILE
    \STATE $\mathcal{R}_{1} = \mathcal{R}_0 \bigcup \mathcal{I}_0$
    \STATE Initialize $\cC_0$.
    \FOR {$s=1,2, \ldots, M-1$}
     \STATE Set $\epsilon_s = c/\sqrt{2^s}$
       \STATE Let $\mX_s$ be the original feature matrix with only the features in $\mathcal{R}_s$.
      \FOR {$t = 1, \ldots, 2^s$}    
      \STATE Draw $b_t$ from bernoulli$(\epsilon_s)$
        \IF{$b_t = 1$}
        \STATE Pick an arm $\x_t$ uniformly at random from $\mathcal{X}$, 
        \ELSE 
        \STATE Pick arm $\x_t$ such that $(\x_t, \tilde{\tta}_t) = \argmax_{(\x, \tta) \in \cX_t \times \cC_{t-1}} \la \x, \tta \ra$
        \ENDIF
        \STATE Play arm $\x_t$ to observe reward $y_t$ and indices revealed for this arm, $\mathcal{I}_t$.
         \STATE Update $\mathcal{R}_{s} = \mathcal{R}_s \bigcup \mathcal{I}_t$
        \IF{$\mathcal{I}_t$ is empty }
        \STATE Rank one update to OFUL confidence set $\cC_{t}$ using $(y_t, \x_t)$
        \ELSE
        \STATE Update $\mX_s$ with features in $\mathcal{R}_s$
        \STATE Recompute the OFUL confidence set $\cC_{t}$ with new feature set $\mX_s$. 
        \ENDIF
           \ENDFOR 
     \STATE  $\mathcal{R}_{s+1} = \mathcal{R}_s$
     \STATE $\cC_1 = \cC_{t}$
    \ENDFOR
  \end{algorithmic}
  \caption{Feature Feedback Epoch OFUL}
  \label{alg:epochHybOFUL}
}
\end{algorithm}

\begin{figure}[H]
\centering
\includegraphics[width= 0.5\linewidth]{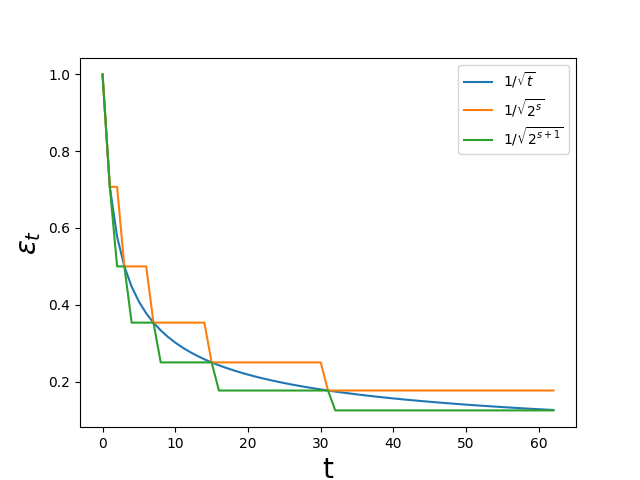}
\caption{Choice of $\epsilon_t$ for both the algorithms. $s = \floor{\log_2 t}$ Recall that $\epsilon_t$. controls the number of pure exploration steps in the algorithms.}
\label{fig:epsilon_greedy}
\end{figure}

\section{Proof of Theorem~\ref{thm:HybOFULRegret}}

We begin by proving intermediate results for three different events followed by the proof details. 
\begin{enumerate}
	\item The number of times we pull a random arm during an epoch is close to its expectation.
	\item We have seen all the relevant arms before the current epoch.
	\item Modified OFUL regret bound using arms from both exploration and exploitation.
\end{enumerate}

\subsection{Bounding the number of times we pull a random arm}
\begin{lem}
During epoch $s$, there are $T_s = 2^s$ time steps. Let $N_s$ be the number of random arm pulls during epoch $s$. Given that the probability of pulling a random arm during epoch $s$ is $\epsilon_s = c/\sqrt{T_s}$, then for any $\delta_1 > 0$:
\begin{align*}
\Pr \left( \lv N_s - c \sqrt{T_s} \rv \geq \sqrt{\frac{T_s}{2} \log \frac{2}{\delta_1}} \right) \leq \delta_1
\end{align*}
\label{lem:num_rand_arms}
\end{lem}

\begin{proof}
We can see $N_s$ as the sum of $T_s$ i.i.d. Bernoulli random variables with probability of success of $\epsilon_s$. It is easy to see that $\E N_s = T_s \cdot c/\sqrt{T_s} = c\sqrt{T_s}$. Finish by applying the Hoeffding's inequality to the sum of the Bernoulli random variables.
\end{proof}

\begin{cor}
With probability $\geq 1 - \delta_1$:
\begin{align*}
\sqrt{\frac{T_s}{2} \log \frac{2}{\delta_1}} \leq N_s \leq 3 \sqrt{\frac{T_s}{2} \log \frac{2}{\delta_1}}
\end{align*}
\label{cor:num_rand_pulls}
\end{cor}

\begin{proof}
This is a simple consequence of taking $c = \sqrt{2 \log \frac{2}{\delta_1}}$ in Lemma~\ref{lem:num_rand_arms}.
\end{proof}

\subsection{Probability of having identified all the relevant arms}
\begin{prop}
Let $\alpha_0 = \sqrt{2}$ and $\alpha_i = \sqrt{\frac{T_s}{2}} = \sqrt{2^{i-1}}$ for $i > 0$. Then:
\begin{align*}
\sum_{i=0}^{s-1} \alpha_i \geq \sqrt{2^s}
\end{align*}
\label{conj:sum_rand}
\end{prop}

\begin{prop}
The number of random arms pulled before an epoch $s$ can be bounded as:
\begin{align*}
\sqrt{2^s \log \frac{2}{\delta_1}} \leq \sum_{i=0}^{s-1} N_i \leq 3\sqrt{2^s \log \frac{2}{\delta_1}}
\end{align*} 
with probability $\geq 1 - s \delta_1$. 
\end{prop}
\begin{proof}
This is a direct result of Corollary~\ref{cor:num_rand_pulls} and Proposition~\ref{conj:sum_rand}.
\end{proof}

\begin{definition}
Let $E^j_s$ be a random variable:
\begin{align*}
E^j_s = \left\{ \begin{array}{cc}
1 & \text{ if the $j$ marked as relevant up till epoch s} \\
0 & \text{ otherwise} \\
\end{array} \right.
\end{align*}
Let $E_s = \displaystyle\bigcap_{j=1}^k E_s^j$ be the event that all the relevant features are marked. 
\end{definition}

\begin{prop}
The probability that we have not seen all the relevant arms goes down quickly. Here we characterize how quickly. Note the assumption here that at every around, we assume that each relevant feature is revealed with some probability at least $p$ independent of other relevant features.

The probability that all the $k$ relevant features have not been marked up to epoch $s$, $\Pr(E_s = 0)$ is bounded as follows.
\begin{align*}
\Pr(E_s = 0)  \leq k \exp\left( - \log \frac{1}{1-p} \displaystyle \sum_{i=0}^{s-1} N_i \right) 
\end{align*}
\label{prop:prob_of_no_obs}
\end{prop}

\begin{proof}
The proof follows by union bound. 
\begin{align*}
\Pr(E_s = 0) = \Pr \left(\bigcup_{j = 1}^k E_s^j = 0 \right) \leq \sum_{j = 1}^k \Pr (E_s^j = 0) \leq k (1 - p)^{\displaystyle \sum_{i=0}^{s-1} N_i}  \\
= k \exp\left( - \log \frac{1}{1-p} \displaystyle \sum_{i=0}^{s-1} N_i \right) 
\end{align*}
\end{proof}

Now we can find the number of epochs that need to pass after which we have observed all the features with high probability:
\begin{prop}
After:
$$s = \left\lceil \log_2 \left( \frac{1}{\log 2/\delta_1} \left( \frac{\log k/\delta_2}{\log 1/(1-p)} \right)^2 \right) \right\rceil \defeq s_{observed}$$
epochs, we have observed all the relevant features with probability $\geq 1 - \delta_2$.
\label{prop:sum_of_sqrts}
\end{prop}

\begin{proof}
\begin{align*}
&\Pr(E_s = 0) \leq k \exp\left( - \log \frac{1}{1-p} \displaystyle \sum_{i=0}^{s-1} N_i \right) \leq k \exp\left( - \log \frac{1}{1-p} \sqrt{2^s \log \frac{2}{\delta_1}} \right) \\
&\leq \delta_2 \text{ , we desire this}\\
&\Rightarrow \exp\left( - \log \frac{1}{1-p} \sqrt{2^s \log \frac{2}{\delta_1}} \right) \leq \frac{\delta_2}{k} \\
&\Rightarrow  \log \frac{1}{1-p} \left( \sqrt{2^s \log \frac{2}{\delta_1}} \right) \geq  \log \frac{k}{\delta_2} \\
&\Rightarrow \sqrt{2^s \log \frac{2}{\delta_1}} \geq \frac{\log k/\delta_2}{\log 1/(1-p)} \\
&\Rightarrow 2^s \geq \frac{1}{\log 2/\delta_1} \left( \frac{\log k/\delta_2}{\log 1/(1-p)} \right)^2 \\
&\Rightarrow s \geq \log_2 \left( \frac{1}{\log 2/\delta_1} \left( \frac{\log k/\delta_2}{\log 1/(1-p)} \right)^2 \right) \\
&\Rightarrow s_{observed} = \left\lceil \log_2 \left( \frac{1}{\log 2/\delta_1} \left( \frac{\log k/\delta_2}{\log 1/(1-p)} \right)^2 \right) \right\rceil  \geq s
\end{align*}
\end{proof}

\subsection{Regret for a modification of OFUL after epoch $s_{observed}$}

We cannot use the OFUL regret bound directly since our algorithm involves additional random arms sampled during the epoch along with arms sampled in previous epochs. To bound the regret of arms pulled using OFUL, we prove the following regret bound for the modified OFUL algorithm, stated as Algorithm~\ref{alg:oful_extended}, where some additional arms are sampled in addition to the OFUL ones:
\begin{lem}\label{lem:oful_extended}
Assume that $\forall t > 0$ and $\x \in \cX_t \subset \R^d$, $\la \x, \tstar \ra \in [-1,1]$. Then with probability at least $1 - \delta$, the regret of Extended OFUL (Algorithm~\ref{alg:oful_extended}) satisfies:
\begin{align*} 
\forall t, R_t \leq & 4 \sqrt{td \log(\lambda + tL/d)} ( \lambda^{1/2} S + R \sqrt{2 \log(1/\delta) + d \log(1 + tL/(\lambda d))} t) 
\end{align*}
where $\lambda > 0$ is the ridge regression parameter of OFUL.
\end{lem}

This lemma shows that the additional arms sampled between of OFUL turns do not harm the regret of OFUL.

\begin{algorithm}[H]
{
  \begin{algorithmic}[1]
  \STATE Begin with some initial arms $\X_0$ which could be empty. 
    \FOR {$t = 1,2, \ldots, T-1$}
    \STATE $(\x_t, \tilde{\tta}_t) = \argmax_{(\x, \tta) \in \cX_t \times \cC_{t-1}} \la \x, \tta \ra$
    \STATE Play $\x_t$ and receive reward $y_t$.
    \STATE Update $\X_t, \overline{\V}_t =(\X_t^T \X_t + \lambda \I)$ and $\hth_t = \overline{\V}_t ^{-1} \X_t^T \y_t$ 
    \STATE Update ellipsoidal confidence set $\cC_t$ as \\
     $\cC_t = \left\{ \th \in \R^d : \|\hth_t - \th\|_{\overline{\V}_t} \leq f(R, \X_t, \y_t, \lambda, \delta, S) \right\}$ (for details on $f(\cdot)$ see \cite{ay11improved})
     \STATE Add some arms (randomly or otherwise) to the set $\X_t$.
    \ENDFOR
  \end{algorithmic}
  \caption{Extended OFUL}
  \label{alg:oful_extended}
}
\end{algorithm}

We will require the following result to prove the theorem. 

\begin{prop}
\label{lem:extended}
For symmetric positive definite matrices $\W$ and $\V = \W+\Q$, we have
\begin{align*}
\|\x\|_{\V^{-1}} \leq \|\x\|_{\W^{-1}}
\end{align*}
\end{prop}
\begin{proof}
Let $\y = \W^{-1/2} x$ and eigenvalue decomposition be $ \W^{-1/2}\Q\W^{-1/2} = \U\boldsymbol{\Sigma} \U^T$. Then we have 
\begin{align*}
\|\x\|_{\V^{-1}}^2 &= \x^T \V^{-1} \x &\\
& = \x^T (\W + \Q)^{-1} \x &\\
&= \y^T \W^{1/2} (\W + \Q)^{-1} \W^{1/2} y&\\
& = \y^T (\I + \W^{-1/2}\Q\W^{-1/2})^{-1} y &\\
&= \y^T (\U\U^T + \U\boldsymbol{\Sigma} \U^T)^{-1} \y &  \U\U^T = \U^T\U= \I\\
&= \y^T \U^T (\I + \boldsymbol{\Sigma})^{-1} \U \y & \\
&\leq \y^T \U^T \I \U \y & \textnormal{ since }\forall i, \frac{1}{1+\sigma_i} \leq 1\\
&\leq \y^T \y &\\
& = \|\x\|_{\W^{-1}}^2 &
\end{align*}
\end{proof}

The remaining proof follows the proof of Theorem 3 and we state it here for the sake of completeness. 
\begin{proof}
Let $\overline{\X}_t = [\x_1^T, \dots, \x_t^T], \overline{\W}_t =(\overline{\X}_t^T \overline{\X}_t   + \lambda \I)$.

We will follow the proof of Theorem 3 in [\cite{ay11improved}] which is divided into 2 parts: first they prove that with high probabiliy $\tstar$ lies inside the confidence set constructed by OFUL at that time.
Notice that the super martingale arguments used to prove that $\tstar$ is inside the confidence set with high probability do not make an assumption on how the previous arms were sampled so the argument goes through without any modification. 

As in \cite{ay11improved} we can decompose the instantaneous regret as follows:
\begin{align*}
r_t & = \langle \tstar, \x_{\ast} \rangle -   \langle \tstar, \x_t \rangle \\
&\leq \langle \tilth_t, \x_t \rangle -   \langle \tstar, \x_t \rangle\\
&=  \langle \tilth_t - \tstar, \x_t \rangle\\
&= \langle \hatth_{t-1} - \tstar, \x_t \rangle + \langle \tilth_t - \hatth_{t-1}, \x_t \rangle\\
&= \| \hatth_{t-1} - \tstar\|_{\overline{\V}_{t-1}^{-1}}\| \x_t \|_{\overline{\V}_{t-1}^{-1}} + \|\tilth_t - \hatth_{t-1}\|_{\overline{\V}_{t-1}^{-1}}\| \x_t \|_{\overline{\V}_{t-1}^{-1}}\\
&\leq 2 \sqrt{\beta_{t-1}(\delta)}\| \x_t \|_{\overline{\V}_{t-1}^{-1}} 
\end{align*}
where we use the fact that $(\tilth_t, \x_t)$ is optimistic and that $\hatth_t, \tilth_t, \tstar $ all lie in the confidence set with high probability. 
Thus with probability at least $1-\delta$, for all $t \geq 0$

\begin{align*}
R_t &\leq \sqrt{t \sum_{s = 1}^t r_t^2} \leq \sqrt{ 8 \beta_t(\delta) t \sum_{s=1}^t \| \x_t \|_{\overline{\V}_{t-1}^{-1}} } \leq \sqrt{ 8 \beta_t{\delta} t \sum_{s=1}^t \| \x_t \|_{\overline{\W}_{t-1}^{-1}} }
\end{align*}
where we used Proposition~\ref{lem:extended} stated above.

By Lemma 11 in \cite{ay11improved} we have,
\begin{align*}
R_t & \leq \sqrt{ 8 \beta_t(\delta) t \sum_{s=1}^t \| \x_t \|_{\overline{\W}_{t-1}^{-1}} }\\
& \leq\sqrt{ 8 \beta_t{\delta} t \log(\textnormal{det}(\overline{\W}_{t}))}\\
&\leq 4 \sqrt{td\log{(\lambda + nL/d)}} \left( \lambda^{1/2}S + R\sqrt{2 \log(1/\delta) + d \log(1+tL/(\lambda d))}\right)
\end{align*}

\end{proof}

\subsection{Regret after epoch $s_{observed}$}
During each epoch after $s_{observed}$, we have at most $3 \sqrt{\frac{T_s}{2} \log \frac{2}{\delta_1}}$ random arm pulls.
\begin{lem}
For epochs $s \geq \s_{observed}$, the cumulative regret is bounded by:
\begin{align*}
&\sum_{s=\s_{observed}}^{M-1} R_s \leq 6 S L\sum_{s=\s_{observed}}^{M-1} \sqrt{\frac{T_s}{2} \log \frac{2}{\delta_1}} \\
&+ 4 \sqrt{T_s k \log(\lambda + n L/k) } \left( \lambda^{1/2} S + R \sqrt{2 \log(1/\delta_3) + k \log(1 + T_s L/(\lambda k))} \right) 
\end{align*}
with probability $\geq 1 - \delta_3 $.
\label{lem:ff_epoch_regret}
\end{lem}

\begin{proof}
The regret during the epoch is the sum of the regret when we pull the random arms added to the regret when we pull OFUL arms. 

Now, we just have to use the upper bound on the number of times we pull a random arm in Corollary~\ref{cor:num_rand_pulls}. During each random arm pull the worst case regret is $2 S L$.

The number of times we pull an OFUL arm in epoch $s$, $T_s^{OFUL}$, is trivially upper bounded by $T_s$. Apply Lemma~\ref{lem:oful_extended} stated above with $\delta \rightarrow \delta_3$, $t \rightarrow T_s^{OFUL}$, $d \rightarrow k$ to get the result. Recall, we cannot apply the OFUL regret bound directly here since our algorithm involves additional random arms sampled during the epoch along with arms sampled in previous epochs. 
\end{proof}

\subsection{Proof of main result}
We are now ready to prove the regret bound of Feature Feedback Epoch OFUL.
\begin{proof}
The regret can be summed over the epochs as:
\begin{align*}
R_T &= \sum_{s = 0}^{M-1} R_s \\
&= \sum_{s = 0}^{s_{observed}} R_s + \sum_{s = s_{observed} + 1}^{M-1} R_s \\
&\leq \sum_{s = 0}^{s_{observed}} 2 S L T_s  +  \sum_{s = s_{observed} + 1}^{M-1} R_s \\
&\leq 2SL 2^{s_{observed} + 1} +  \sum_{s = s_{observed} + 1}^{M-1} R_s
\end{align*}

Now, note that:
\begin{align*}
2^{s_{observed} + 1} &= 2 \cdot 2^{\left\lceil \log_2 \left( \frac{1}{\log 2/\delta_1} \left( \frac{\log k/\delta_2}{\log 1/(1-p)} \right)^2 \right) \right\rceil} \\
&\leq 2 \cdot 2^{\log_2 \left( \frac{1}{\log 2/\delta_1} \left( \frac{\log k/\delta_2}{\log 1/(1-p)} \right)^2 \right) + 1} \\
&= \frac{4}{\log 2/\delta_1} \left( \frac{\log k/\delta_2}{\log 1/(1-p)} \right)^2 
\end{align*}

Now, setting $\delta_1 = \delta_3 = \delta/3M$ and $\delta_2 = \delta/3$, we get the final regret expression using Lemma~\ref{lem:ff_epoch_regret}. The multiplicative factor of $\log\frac{T}{2}$ comes from bounding the sum of regrets over the epochs by the max regret over all the epochs (which occurs during the last epoch) multiplied by the number of epochs, which is $\log\frac{T}{2}$.

\end{proof}

The proof for Feature Feedback OFUL follows similarly by noticing that the Algorithms are essentially the same with different $\epsilon_t$ and using the fact that $ \frac{1}{\sqrt{2^{s+1}}} \leq \epsilon_t  = \frac{1}{\sqrt{t}} \leq \frac{1}{\sqrt{2^s}}$ for $s = \floor{\log_2(t)}$.

\end{document}

%% file: packs.tex
\usepackage{amsmath,amssymb}
\usepackage[mathscr]{euscript} 
\usepackage{amsthm} 
\usepackage{graphicx} 
\usepackage{framed} 
\usepackage{float} 
\usepackage{bbm} 
\usepackage{multirow}
\usepackage{color}
\usepackage{subfig}
\usepackage{longtable}
\usepackage{xspace}
\usepackage{wrapfig}
\usepackage{lipsum}
\usepackage{empheq}
\usepackage{setspace}
\usepackage[hidelinks, breaklinks=true]{hyperref} 
\usepackage[font=footnotesize]{caption}

\usepackage{times,verbatim}
\usepackage{epsfig} 

\usepackage{color}
\usepackage{algorithm}
\usepackage{algorithmic}
\usepackage{array, mathtools, bbold, bm,pifont,dsfont,authblk}

\usepackage{multicol}
\usepackage{titlesec}
\usepackage{wasysym}

\usepackage{fancyhdr}



%% file: symbs.tex
\newcommand{\cC}{\mathcal{C}}

\newcommand{\Real}{\mathbb{R}}

\newcommand{\x}{\vct{x}}
\newcommand{\y}{\vct{y}}

\newcommand{\beq}{\begin{equation}}

\newcommand{\eeq}{\end{equation}}



\newcommand{\defeq}{\mathrel{\mathop:}=}

\renewcommand{\Pr}{\mathbb{P}}

\newtheorem{theorem}{Theorem}[section]

\newtheorem{definition}[theorem]{Definition}

\def\la{{\langle}}
\def\ra{{\rangle}}
\def\th{\ensuremath{\boldsymbol{\theta}}\xspace}

\def\hatth{\ensuremath{\widehat{\boldsymbol\theta}}\xspace}

\def\tilth{\ensuremath{\tilde{\boldsymbol\theta}}\xspace}

\def\I{\ensuremath{\mathbf{I}}}

\def\X{\ensuremath{\mathbf{X}}}

\def\cX{\ensuremath{\mathcal{X}}\xspace}

\def\E{\mathbf E}

\def\hth{\ensuremath{\what{\theta}}\xspace}

\newtheorem*{lemma*}{Lemma}
\newtheorem*{theorem*}{Theorem}
\newtheorem*{corollary*}{Corollary}
\newtheorem*{proposition*}{Proposition}

\usepackage{hyperref}
\usepackage{url}

\usepackage{dsfont} 
\usepackage{xspace, amsmath, amssymb, amsthm}
\usepackage{color}

\pdfoutput=1

\usepackage{graphicx}
\usepackage{tabularx}

\usepackage{algorithm}
\usepackage{algorithmic}

\usepackage{multirow}
\usepackage{hhline}

\usepackage{booktabs}

\usepackage{capt-of}
\newtheorem{lem}{Lemma}
\newtheorem{thm}{Theorem}
\newtheorem{ass}{Assumption}
\newtheorem{cor}{Corollary}

\newtheorem{prop}{Proposition}

\def\argmax{\ensuremath{\mbox{argmax}}}

\allowdisplaybreaks 

\def\Pr{\ensuremath{\mbox{Pr}}}

\def\supp{\ensuremath{\mbox{supp}}\xspace}

\newcommand{\what}[1]{ {\ensuremath{\widehat{#1}}} }

\usepackage[export]{adjustbox}

\usepackage{pbox} 

\makeatletter
\newcommand{\vast}{\bBigg@{3}}
\newcommand{\Vast}{\bBigg@{4}}
\makeatother

\def\cX{\ensuremath{\mathcal{X}}\xspace} 
\def\th{\ensuremath{\boldsymbol{\theta}}\xspace} 

\def\x{{{\mathbf x}}}
\def\y{{{\mathbf y}}}

\def\la{{\langle}}
\def\ra{{\rangle}}

\def\hatth{\ensuremath{\widehat{\boldsymbol\theta}}\xspace} 
\def\tilth{\ensuremath{\tilde{\boldsymbol\theta}}\xspace} 
 
\def\U{\ensuremath{\mathbf{U}}\xspace} 
 
\def\V{\ensuremath{\mathbf{V}}\xspace}

\usepackage{pifont}
\def\cC{\ensuremath{\mathcal{C}}\xspace}

\usepackage{enumitem}

\def\W{\ensuremath{\mathbf{W}}}

\def\X{\ensuremath{\mathbf{X}}} 
\def\I{\ensuremath{\mathbf{I}}}

\newcommand{\mtx}[1]{\bm{#1}}
\newcommand{\tta}{\th}
\newcommand{\tstar}{\th_*}

\def\th{\ensuremath{\theta}\xspace}
\def\hth{\ensuremath{\what{\theta}}\xspace}
 
\def\th{{{\boldsymbol \theta}}}

\def\hth{{\what{\boldsymbol \theta}}} 

\def\X{\ensuremath{\mathbf{X}}\xspace}
\def\U{\ensuremath{\mathbf{U}}\xspace}
\def\x{\ensuremath{\mathbf{x}}\xspace}

\def\s{\ensuremath{\mathbf{s}}\xspace}

\def\R{\ensuremath{\mathbf{R}}} 
\def\Q{\ensuremath{\mathbf{Q}}} 
\def\E{\ensuremath{\mathbf{E}}}

\usepackage{accents}
\newlength{\dhatheight}

\usepackage[export]{adjustbox}

\newcommand{\lv}{\left\vert}
\newcommand{\rv}{\right\vert}

\newcommand{\lV}{\left\Vert}
\newcommand{\rV}{\right\Vert}

\renewcommand{\Pr}{\mathbb{P}}

\newcommand{\mX}{\mtx{X}}